\documentclass{article}
\usepackage{ijcai19}
\usepackage{times}
\usepackage{soul}
\usepackage{setspace}
\usepackage{url}
\usepackage[small]{caption}
\usepackage{subfigure,graphicx}
\usepackage{amsmath}
\usepackage{booktabs}
\usepackage{algorithm}
\usepackage{algorithmic}
\usepackage{float}
\usepackage{amsmath}
\usepackage{amsthm}
\usepackage{appendix}
\newtheorem{theorem}{Theorem}
\newtheorem{lemma}{Lemma}
\newtheorem{definition}{Definition}

\newtheorem{corollary}{Corollary}

\title{On Privacy Protection of Latent Dirichlet Allocation Model Training}

\author{
	Fangyuan Zhao$^{1,2,3}$, Xuebin Ren$^{1,2}$, Shusen Yang$^{2,3}$, Xinyu Yang$^{1,2}$
	\affiliations
	$^{1}$School of Computer Science and Technology, Xi'an Jiaotong University, China\\
	$^{2}$National Engineering Laboratory for Big Data Analytics, Xi'an Jiaotong University, China\\
	$^{3}$Ministry of Education Key Lab For Intelligent Networks and Network Security, Xi'an Jiaotong University, China\\
	\emails
	zfy1454236335@stu.xjtu.edu, xuebinren@mail.xjtu.edu, shusenyang@mail.xjtu.edu, yxyphd@mail.xjtu.edu 
}

\begin{document}
	
	\maketitle
	
	\begin{abstract}
		
		Latent Dirichlet Allocation (LDA) is a popular topic modeling technique for discovery of hidden semantic architecture of text datasets, and plays a fundamental role in many machine learning applications. However, like many other machine learning algorithms, the process of training a LDA model may leak the sensitive information of the training datasets and bring significant privacy risks. To mitigate the privacy issues in LDA, we focus on studying privacy-preserving algorithms of LDA model training in this paper. In particular, we first develop a privacy monitoring algorithm to investigate the privacy guarantee obtained from the inherent randomness of the collapsed gibbs sampling (CGS) process in a typical LDA training algorithm on centralized curated datasets. Then, we further propose a locally private LDA training algorithm on crowdsourced data to provide local differential privacy for individual data contributors. The experimental results on real-world datasets demonstrate the effectiveness of our proposed algorithms. 
		
	\end{abstract}
	
	\section{Introduction}
	
	Massive text data have arisen in the sustained and rapid development of Internet. Mining and analyzing of text data can help us gain a vast amount of knowledge, thus benefiting the whole society. As a fundamental model for text mining, Latent Dirichlet Allocation(LDA)~\cite{blei2003latent} can be used for discovering the main features of the sparse text datasets by identifying their hidden semantic architecture. Particularly, LDA can map the high-dimensional text data to a low-dimensional topic space while retaining the implicit semantics, which has been an effective machine learning technique for clustering or classification. Many enterprises such as Yahoo~\cite{smola2010architecture}, Tencent~\cite{wang2014towards}\cite{yut2017lda}, and Microsoft~\cite{yuan2015lightlda} have all built LDA platforms for supporting big data analysis and training machine learning models on various text data.
	
	Similar to other machine learning models, LDA may be trained on the datasets that contain some sensitive information of individuals and will inevitably memorize some knowledge about the datasets. Unfortunately, aiming at this characteristic, some attacks have been proposed to extract the private information of the training data from machine learning models. For example, membership inference attacks (MIA)\cite{shokri2017membership} can be launched to infer the membership information of an individual. Model inversion attacks \cite{fredrikson2014privacy} have been proved to be able to extract training data from observed model predictions. Therefore, despite the popularity and effectiveness, the naive LDA model may also suffer from these attacks and lead to great privacy risks. 
	
	Differential privacy proposed by Dwork~\cite{dwork2006calibrating} has been the de-facto standard of privacy protection with a rigorous mathematical proof. Due to its strong privacy guarantee, DP has also been exploited in many fields such as data publication~\cite{ren2018textsf}\cite{li2019impact}and machine learning~\cite{chaudhuri2011differentially}\cite{abadi2016deep} as well as LDA training~\cite{park2016variational}\cite{zhu2016privacy}. For example, Park et al.~\cite{park2016variational} proposed to obtain privacy guarantee for LDA models by perturbing the expected sufficient statistics in each iteration of the variational Bayesian method, which is a parameter estimation algorithm for LDA. Zhu et. al.~\cite{zhu2016privacy} presented a differentially private LDA algorithm by perturbing the sampling distribution in the collapsed Gibbs sampling(CGS) process, which is a typical training algorithm for LDA.
	
	Both the above algorithms achieve DP by injecting extra noise to the training process of LDA regarding to centralized training datasets. However, as a typical sampling algorithm with inherent randomness, CGS possesses uncertainty in its execution and naturally provides some level of privacy guarantee, which has been indicated in \cite{wang2015privacy}\cite{foulds2016theory}. In particular, Wang et al.~\cite{wang2015privacy} proved that posterior sampling and the stochastic gradient Markov chain Monte Carlo techniques possess some inherent privacy guarantee. Foulds et al.~\cite{foulds2016theory} further extended this conclusion to the general MCMC methods. Besides the inherent privacy, both existing algorithms consider the LDA model training on centralized datasets owned by a trustworthy data curator. Nevertheless, due to privacy concerns, individual data contributors may be reluctant to directly share their sensitive data but prefer to send the locally sanitized data to the model trainer. 
	
	Therefore, aiming to provide strong privacy guarantee for LDA model training, this paper not only investigates to utilize the inherent privacy of CGS in LDA training on centralized datasets, but also proposes a locally private version of LDA that can be trained on crowd-sourced datasets with local sanitations. The contributions are summarized as follows:
	\begin{itemize}
		\item We develop a privacy monitoring algorithm to measure the inherent privacy guarantee of CGS algorithm in LDA. In particular, we first define two different levels of privacy: document level and word level, and present the corresponding lower bound of privacy guarantee after a given number of iterations.
		\item We propose LP-LDA, a novel mechanism that supports training a LDA model on crowd-sourced datasets with local sanitation, which can provide the guarantee of local privacy for individual data contributors.
		\item We conduct experiments on several real-world datasets to demonstrate the effectiveness of our proposed algorithms. Particularly, experimental results show that our LP-LDA can achieve a high model training accuracy while providing sufficient local privacy guarantee.
	\end{itemize}
	
	The rest of paper is organized as follows. Section 2 reviews the preliminaries. Section 3 describes our algorithms in detail. The experiments are presented in Section 4. Finally, we conclude the paper in Section 5.

	\section{Preliminaries}
	
	\subsection{LDA and Collapsed Gibbs Sampling}
	
	LDA model was first proposed by David Blei~\cite{blei2003latent} in 2003 for analyzing the implicit semantic architecture of a corpus. In LDA model, any document $m$ in a corpus $D$ can be described by different distributions on $K$ latent topics, where each topic $k$ can be represented by a distribution on all words. LDA assumes the generative process of the documents in the corpus $D$ as follows:
	\begin{enumerate}
		\item for each topic $k$, draw a ``topic-word'' distribution $\phi_k$ on all words $t$ from Dirichlet($\beta\vec{1}$), where $\beta$ is the hyperparameter for the Dirichlet priors and can be interpreted as the prior observation for the ``topic-word'' count.
		\item for each document $m$, draw a ``document-topic'' distribution $\theta_m$ from Dirichlet($\alpha\vec{1}$), where $\alpha$ is a hyperparameter similar to $\beta$ and represents the prior observation for the ``document-topic'' count.
		\item for each word $w$ in a document $m$, first draw a topic \emph{k} from $\theta_m$, and then draw a word \emph{t} from $\phi_k$.
	\end{enumerate}
	
	The essence of training a LDA model is to estimate the parameters $\phi_k$ for a given corpus $D$. The collapsed Gibbs sampling is such an effective parameter estimation algorithm. It iterates over each word $w_i$ and samples new topic $z_i$ for $w_i$ based on this full conditional distribution
	\begin{equation}\label{sampling equation}
	p(z_{i}=k|\vec{z}_{\neg i},\vec{w})\propto\frac{n_{k}^{t}+\beta}{\sum_{t=1}^{V}(n_{k}^{t}+\beta)}\cdot\frac{n_{m }^{k}+\alpha}{\sum_{k=1}^{K}(n_{m}^{k}+\alpha)}
	\end{equation}
	where $\neg i$ denotes the whole words except word $w_i$, $n_{k}^{t}$ denotes the count of topic $k$ assigned to word $t$ and $n_{m }^{k}$ denotes the count of topic $k$ appeared in document $m$ which are maintained in matrices $N_{k}^{t}$ and $N_{m }^{k}$ respectively.
	
	After multiple rounds of sampling over the whole corpus, the topic sample of each word can be obtained. And the parameter $\phi_k$ can be estimated by its posterior expectation
	\begin{equation*}
	\mathbf{E}[\phi_k^t|\vec{z},\vec{w}]=\frac{n_{k}^{t}+\beta}{\sum_{t=1}^{V}(n_{k}^{t}+\beta)}
	\end{equation*}
	The detailed procedures of CGS can be referred to~\cite{heinrich2005parameter}.

	\subsection{Differential Privacy}
	
	Differential privacy proposed by Dwork~\cite{dwork2006calibrating} has been the de-facto standard of privacy protection with a rigorous mathematical proof. The rationale of DP guarantee is that negligible information can be gained by manipulating the output of a query on neighboring datasets.
	
	\begin{definition}
		\cite{dwork2006calibrating} (Differential Privacy) A randomized mechanism $M:D\rightarrow Y$ is $\varepsilon\text{-differential private}$  if for any neighboring datasets $D, D^{'}$ that satisfying $|D \Delta D'|=1$ and any output $S\subseteq Y$:
		\begin{equation*}
		Pr[M(D)\in S]\leq e^\varepsilon \cdot Pr[M(D^{'})\in S]
		\end{equation*}
	\end{definition}

	\subsection{Local Privacy}
	
	Differential privacy implicitly assumes a centralized dataset owned by a trustworthy curator and does not ensure the privacy guarantee for individual data contributors. Recently, local (differential) privacy has been proposed to provide data sanitization at the individual users' side instead of the central server side.
	\begin{definition}
		\cite{dwork2014algorithmic}(Local Privacy)\emph{ A randomized function $f$ satisfies $\varepsilon$-local privacy if and only if for any two input tuples $t$ and $t'$ in the domain of $f$, and for any output $t^{*}$ of $f$, there is:}
		\begin{equation*}
		Pr[f(t)=t^{*}]\leq e^\varepsilon Pr[f(t')=t^{*}]
		\end{equation*}
	\end{definition}
	
	\section{Our Approach}
	In this section, we first investigate the inherent privacy of CGS process in LDA for a non-sanitized dataset owned by a trustworthy curator. Then, as a complement to the privacy guarantee of the data acquisition period, a locally private mechanism LP-LDA is presented to realize LDA model training on a sanitized dataset by local users. 
	
	\subsection{Privacy monitoring Algorithm}
	
	\subsubsection{Inherent Privacy of CGS}
	
	Generally, DP is achieved on most machine learning algorithms by introducing extra noise or randomness, which will inevitably cause a utility loss of the trained model. However, it has been shown in~\cite{foulds2016theory} that some degree of inherent DP can be obtained on Gibbs sampling algorithm for free. This is because each sampling process in Gibbs sampling works in a way the same as an exponential mechanism, which is a classic method to achieve DP. Obviously, as one version of Gibbs sampling, collapsed Gibbs sampling naturally inherits this property. Furthermore, such a property can also provide privacy for free in the CGS-based LDA training process. Therefore, aiming to utilize the inherent privacy, we develop a privacy monitoring algorithm to quantify the privacy guarantee of CGS in the LDA training process. In particular, the rationale behind the privacy monitoring algorithm is to find an adequate exponential mechanism for each sampling process in CGS and then accumulate the total privacy guarantee of all exponential mechanisms according to the composition theorem of DP.
	
	\subsubsection{Document-level privacy and word-level privacy}
	
	This paper considers to provide DP for the individual words and documents in the training corpus for LDA, respectively. 
	
	\emph{Word-level privacy}: Let $D=\{w_1,w_2,...w_W\}$ denote a corpus with $|D|=W$ words $w_i~(i=1,2,...,W)$. Then, its neighboring dataset $D'$ satisfying $|D \Delta D'|=1$ differs from $D$ by a single word $w$. Word-level privacy prevents membership inference of individual words of the training corpus from the trained LDA model.
	
	\emph{Document-level privacy}: Let $D=\{m_1,m_2,...m_M\}$ denote a corpus with $|D|=M$ documents $m_i~(i=1,2,...,M)$. Then, its neighboring dataset $D'$ satisfying $|D \Delta D'|=1$ differs from $D$ by a single document $m$. In order to bound the sensitivity, we assume that a single document includes at most $N_{max}$ words. Document-level privacy prevents re-identification of individual documents in the training dataset of LDA, which may be contributed by and associated with individual users.
	
	\subsubsection{Inherent privacy in each sampling}
	To begin with, we show the essence of the intrinsic privacy guarantee in each sampling of CGS in terms of exponential mechanism. Consider the sampling process for word $w_i$ in the $n$th iteration. Suppose its sampling distribution on $K$ topics is given by $\mathbf{P}=(p_1,p_2,...,p_K)^{\top}$, where $p_k$ denotes the probability that topic $k$ is assigned to $w_i$ in this sampling. Then we can rewrite $p_k$ as  
	\begin{equation*}
	p_{k}=e^{\frac{(2\Delta\ln p_{k})\ln p_{k}}{2\Delta \ln p_{k}}},
	\end{equation*}
	which could be understood as an output probability of an exponential mechanism $M_E (w_i, u, \mathcal{K})$ that selects the topic $k \in \mathcal{K}$ with probability of $p_k$. The utility function  of $M_E (w_i, u, \mathcal{K})$ is $u(w_i,k)=\ln p_k$ and its sensitivity is $\Delta \ln p_k$. Obviously, $\varepsilon =2\Delta \ln p_{k}$ is the intrinsic privacy guarantee of the exponential mechanism $M_E (w_i, u, \mathcal{K})$.

	\subsubsection{Privacy monitoring for each sampling}
	Unfortunately, it's intractable to specify an exact value of $2\Delta \ln p_{k}$ in the execution process of CGS algorithm in LDA due to the complicated architecture of training corpus, hence we attempt to find an upper bound of $2\Delta \ln p_{k}$ to quantify the privacy guarantee $\varepsilon $.
	
	According to Equation (1), the sampling distribution $\mathbf{P}$ for word $w_i=t$ in $D$ in the $n$-th iteration could be computed by
	\begin{equation*}
	\begin{aligned}
	p_k\propto r_k=\frac{n_{k}^{t}+\beta}{\sum_{t=1}^{V}(n_{k}^{t}+\beta)}\cdot\frac{n_{m}^{k}+\alpha}{\sum_{k=1}^{K}(n_{m}^{k}+\alpha)}
	\end{aligned}
	\end{equation*}
	Suppose that $\mathbf{P'}=(p'_1,p'_2,...,p'_K)^{\top}$ is the corresponding distribution on $D'$, which is the neighboring dataset of $D$, then
	\begin{equation}\label{neiboring sampling equation}
	\begin{aligned}
	p'_k\propto r' _k=\frac{n_{k}^{t}+\beta}{\sum_{t=1}^{V}(n_{k}^{t}+\beta)-N_k}\cdot\frac{n_{m}^{k}+\alpha}{\sum_{k=1}^{K}(n_{m}^{k}+\alpha)}
	\end{aligned}
	\end{equation}
	where $N_k$ denotes the count of topic $k$ assigned in the $D-D'$ where $k \in \{1,2,...,K\}$. We refer to $\{N_1, N_2,...,N_k\}$ as a topic partition on $D-D'$ and $\sum N_k=|D-D'|$.
	
	Given a topic partition $\gamma=\{N_1, N_2,...,N_k\}$, the privacy guarantee in this sampling process could be measured by
	\begin{equation*}
	\varepsilon_{\gamma}=\max_{k\in \{1,2,...,K\}}\{2\xi_{k}\}=\max_{k\in \{1,2,...,K\}}\{2\ln\frac{{p'}_k}{p_k}\}
	\end{equation*}
	where $\xi_k$ denotes the sensitivity of $\ln{p_{k}}$. 
	However, there are $\binom{N+K-1}{K-1}$ partitions in total. So, it is computational prohibitive to find the maximal $\varepsilon_{\gamma}$ among all partitions. In the following, we consider how to reduce the searching space of partitions.
	
	For simplicity, we first consider a special case, in which there exists some topic \emph{i} with $N_i=0$ in a given partition.
	
	\begin{theorem}\label{theorem:th1}
		Suppose that there exists some $N_k=0$ in a given partition $\gamma=\{N_1,N_2,...,N_K\}$, then the privacy guarantee
		\begin{equation}\label{equ:result of th1}
		\varepsilon_{\gamma}=2\xi_k=2\max\{\xi_1\,\xi_2\,...,\xi_K\}=2\ln\frac{\sum_{k}r_{k}^{'}}{\sum_{k}r_k},
		\end{equation}
		\emph{if and only if for any $j\neq k$}
		\begin{equation}\label{condition of th1}
		\ln\frac{\sum_{t=1}^{V}(n_{k}^{t}+\beta)}{\sum_{t=1}^{V}(n_{k}^{t}+\beta)-N_j}<2\ln\frac{\sum_{k}r_{k}^{'}}{\sum_{k}r_k}
		\end{equation}
	\end{theorem}
	
	\begin{proof}
		See Appendix A for details.
	\end{proof}
	
	\begin{corollary}\label{cor:cor1}
		Suppose that there exists a topic set $\mathcal{T}=\{k,...,j\}$ with $\{N_j\neq 0,\forall j\in \mathcal{T}\}$ in a given partition $\gamma$, and it holds that
		\begin{equation*}
		\ln\frac{\sum_{t=1}^{V}(n_{k}^{t}+\beta)}{\sum_{t=1}^{V}(n_{k}^{t}+\beta)-N_k}>2\ln\frac{\sum_{k}r_{k}^{'}}{\sum_{k}r_k}
		\end{equation*}
		\emph{for some $k \in \mathcal{T}$, then the privacy guarantee}
		\begin{equation}
		\begin{aligned}
		\varepsilon_{\gamma}=2\max_{k\in \mathcal{T}}\{\ln(\frac{\sum_{k}r'_k}{\sum_{k}r_k}\cdot\frac{r_k}{r'_{k}})\}>2\ln\frac{\sum_{k}r'_{k}}{\sum_{k}r_k}
		\end{aligned}
		\end{equation}
	\end{corollary}
	\begin{proof}
		This proof follows from the result of Theorem 1.
	\end{proof}
	Theorem 1 and corollary 1 illustrate a special case to find the privacy $\varepsilon_{\gamma}$. The following lemma and theorem further demonstrate that among all the partitions, the one with the largest privacy guarantee belongs to a partitions set $\mathcal{P}=\{\gamma|\exists k,s.t. ~N_k=N,N_j=0,\forall j\neq k\}$.
	
	\begin{algorithm}[tb]
		\caption{Privacy Monitoring for Each Sampling}
		\label{alg:algorithm1}
		\textbf{Input}: word count matrices  $N_k^t$ and $N_m^k$, $N=|D-D'|$ ($D=1$ for word-level privacy or $D=N_{max}$ for document-level privacy)\\
		\textbf{Parameter}: hyper parameters $\alpha, \beta$\\
		\textbf{Output}: privacy guarantee $\varepsilon$
		\begin{algorithmic}[1] 
			\STATE compute the sampling distribution $\mathbf{P}$ with 
			$
			p_k\propto r_k=\frac{n_{k}^{t}+\beta}{\sum_{t=1}^{V}(n_{k}^{t}+\beta)}\cdot\frac{n_{m}^{k}+\alpha}{\sum_{k=1}^{K}(n_{m}^{k}+\alpha)}
			$
			\STATE compute pseudo sampling distribution $\mathbf{q}$ with 
			$
			q_k=\frac{n_{k}^{t}+\beta}{\sum_{t=1}^{V}(n_{k}^{t}+\beta)-N}\cdot\frac{n_{m}^{k}+\alpha}{\sum_{k=1}^{K}(n_{m}^{k}+\alpha)}
			$
			\FOR{each component $q_k$ of $\mathbf{q}$}
			\STATE compute $\xi_k = \ln(\frac{\sum_{k}r_{k}^{'}}{\sum_{k}r_k}\cdot\frac{r_k}{q_{k}})$
			\ENDFOR
			\STATE find index $k$ such  that $|r_k-q_k|=\left\|\mathbf{r}-\mathbf{q}\right\|_\infty $
			\STATE compute $\xi=\ln(\frac{\sum_{j\neq k}r_j+q_k}{\sum_{j}r_j})$ with $q_{k}$
			\STATE \textbf{return} $\varepsilon=2\max\{\xi_1,\xi_2,...,\xi_K,\xi\}$
		\end{algorithmic}
	\end{algorithm}
	
	\begin{lemma}\label{lemma:lemma1}
		There exists a partition $\gamma^{*}$ in
		\begin{equation}
		\mathcal{P}=\{\gamma|\exists k,s.t N_k=N,N_j=0,\forall j\neq k\}
		\end{equation}
		such that
		\begin{equation*}
		\sum_{\gamma^{*}}r'_{k}=\max_{\gamma}\{\sum_{k}r'_{k}|\gamma\in\Gamma\}
		\end{equation*}
		where $\Gamma$ denotes the set consisting of all the partitions.
	\end{lemma}
	\begin{proof}
		See Appendix B for details.
	\end{proof}
	
	\begin{definition}
		(Pseudo sampling distribution) Suppose that given a vector $\mathbf{q}$ with length K, each component
		\begin{equation}
		q_k=\frac{n_{k}^{t}+\beta}{\sum_{t=1}^{V}(n_{k}^{t}+\beta)-N}\cdot\frac{n_{m}^{k}+\alpha}{\sum_{k=1}^{K}(n_{m}^{k}+\alpha)}
		\end{equation}\label{Pseudo sampling distribution}
		Then $\mathbf{q}$ is the pseudo sampling distribution in this sampling.
	\end{definition}
	
	\begin{theorem}\label{theorem:th2}
		Among all the partitions, there must exist a partition $\gamma'$ in 
		\begin{equation*}
		\mathcal{P}=\{\gamma|\exists k,s.t N_k=N,N_j=0,\forall j\neq k\}
		\end{equation*}
		such that 
		\begin{equation}
		\varepsilon_{\gamma'}=\max_{\gamma}\{\varepsilon|\gamma\in\Gamma\}=2\ln(\frac{\sum_{j\neq k}r_j+q_k}{\sum_{j}r_j})
		\end{equation}\label{result of th2}
		if condition (4) holds. $k$ is the topic index such that $|r_k-q_k|=\left\|\mathbf{r}-\mathbf{q}\right\|_\infty $, $\mathbf{q}$is the pseudo sampling distribution.
	\end{theorem}
	
	\begin{proof}
		See Appendix C for details.
	\end{proof}
	
	Theorem~\ref{theorem:th2} indicates that only the partitions in $\mathcal{P}$ need to be considered for computing the privacy $\varepsilon$ in the each sampling processing, which greatly reduce the searching scope. In particular, if condition (4) holds for all partitions in $\mathcal{P}$, the privacy guarantee could be computed directly by Equation (8), which is the first case to consider.
	If not, for any partition $\gamma$ in $\mathcal{P}$ not satisfying condition (4), the privacy guarantee could be computed by Equation (5). Due to the arbitrariness of $\gamma$, we have another $K$ cases to consider since there are $K$ partitions in $\mathcal{P}$. Furthermore, since whether condition (4) holds is unknown, we have to enumerate all these $K+1$ cases to find the privacy guarantee bound. Algorithm 1 presents the searching-based algorithm for monitoring the privacy guarantee of each sampling for each word.
	
	\subsubsection{Privacy monitoring for LDA}
	
	So far, the privacy guarantee $\varepsilon_w^i$ of the sampling process for word $w$ in the $i$th iteration can be measured by Algorithm~\ref{alg:algorithm1}. Since the sampling process of the whole CGS algorithm is iteratively performed for each word but alternatively among all the words in the corpus, the total privacy guarantee of the whole CGS process in LDA training could be computed according to the composition theorems of DP.
	
	\begin{theorem}
		Given a corpus $D$, suppose the CGS algorithm performed on word $w$ at the $i$-th iteration satisfies $\varepsilon_w^i$-DP, then after $n$ iterations, the whole CGS algorithm performed on $D$ satisfies $ \max_{w}\{\sum_{i=1}^{n}\varepsilon_w^i\}$-DP. 
		\begin{proof}
			For any word $w$, after $n$ iterations of sampling, it will be accessed to by the whole CGS process $n$ times, according to the sequential composition theorem~\cite{li2016differential}, the total privacy guarantee for word $w$ in the CGS algorithm is $\varepsilon_w=\sum_{i=1}^{n}\varepsilon_w^i$. While, according to the Equation (1), each iteration of CGS in LDA only accesses to each word once to perform the sampling, then according to the parallel composition theorem~\cite{li2016differential}, the total privacy guarantee for the copus(all words) should be the maximum privacy guarantee of CGS among all words, that is $ \max_{w}\{\sum_{i=1}^{n}\varepsilon_w^i\}$.
		\end{proof}
	\end{theorem}
	
	Based on this observation, Algorithm~\ref{alg:algorithm2} shows the privacy monitoring algorithm for the whole CGS process in LDA.
	
	\begin{algorithm}[tb]
		\caption{Privacy Monitoring for CGS in LDA}
		\label{alg:algorithm2}
		\textbf{Input}: document corpus \emph{D}\\
		\textbf{Parameter}: iteration number \emph{n}\\
		\textbf{Output}:privacy guarantee $\varepsilon$
		\begin{algorithmic}[1] 
			\WHILE{not finished}
			\FOR{each document \emph{m} in \emph{D}}
			\FOR{each word $w$ in \emph{m}}
			\STATE compute the sampling distribution $\mathbf{p}$
			\STATE  call algorithm~\ref{alg:algorithm1} to compute  $\varepsilon$$_w^i$
			\STATE sample a topic and update matrices N$_k^t$ and N$_m^k$
			\ENDFOR
			\ENDFOR
			\ENDWHILE
			\STATE \textbf{return} $\max_{w}\{\sum_{i=1}^{n}\varepsilon_w^i\}$
		\end{algorithmic}
	\end{algorithm}

	\subsection{LP-LDA}
	As analyzed above, CGS algorithm can intrinsically guarantee the privacy of individual documents for the LDA model trained on a plain-text dataset, which is owned by a trustworthy curator. However, in many distributed applications, data servers are not always privacy-reliable and data owners may not be willing to directly contribute their sensitive data.
	In this case, we further propose a hidden-data based LDA mechanism LP-LDA that can perform the training process on a sanitized dataset with local privacy.  In particular, the LP-LDA mechanism mainly consists of two components: local perturbation at the user side and training on reconstructed dataset at the server side.
	
	\subsubsection{Local perturbation}
	The local perturbation at the user side includes the following steps:
	\begin{itemize}
		\item \textit{Step 1}. Each document $m$ is encoded as a binary vector $\mathbf{V}_m$, in which each bit $\mathbf{V}_m [j]$ represents the presence of the $j$-th word in the word bag of the corpus.
		\item \textit{Step 2}. Each bit $\mathbf{V}_m [j]$ of the binary vector $\mathbf{V}_m$ is then randomly flipped according to the following randomized response rule:
		\begin{align*}
		\hat{\mathbf{V}}_m [i]=
		\begin{cases}
		\mathbf{V}_m [j], ~~& \text{with probability of}~ 1-f\\
		1, ~~& \text{with probability of}~ f/2\\
		0, ~~& \text{with probability of}~ f/2\\
		\end{cases}
		\end{align*}
		where $f \in [0,1]$ is a parameter that specifies the randomness of flipping and adjusts the local privacy level.
		\item \textit{Step 3}. Then the noisy binary vector $\hat{\mathbf{V}}_m [j]$ is sent to the central server by each user. Obviously, $\hat{\mathbf{V}}_m [j]$ is locally sanitized without concerning user's privacy.
	\end{itemize}

	\subsubsection{Training on reconstructed dataset}
	After receiving the flipped binary vectors from a large number of data contributors, the central server can aggregate the vectors, reconstruct the dataset and then perform training on the reconstructed dataset. The rationale behind this is that the training result of topic-word distribution is insensitive to the document partitions and only depends on the total word counts in the corpus.
	\begin{itemize}
		\item \textit{Step 1}. For each bit in the noisy binary vectors, the server counts the number of $1'$s as $n_t=\sum_{i=1}^M \hat{\mathbf{V}}_m [t]$.
		\item \textit{Step 2}. The server then estimates the true count $N_t$ of each bit in the original binary vectors $\mathbf{V}_m$  as $\hat{N}_t=(2n_t-f M)/2(1-f)$.
		\item \textit{Step 3}. For each bit, the server first computes the difference $\delta_t=\hat{N}_t-n_t$.
		\item \textit{Step 4}. For each bit $t$, if $\delta_t>0$, the server randomly samples $\delta_t$ binary vectors with the $t$-th bit as $0$ and sets the $t$-th bit as $1$; if $\delta_t<0$, then the server randomly samples $|\delta_t|$ binary vectors with the $t$-th bit as $1$ and sets the $t$-th bit as $0$; otherwise, keeps the noisy bit vectors as received.
		\item \textit{Step 5}. Based on the noisy bit vectors, the server reconstructs a dataset and performs the CGS process on it.
	\end{itemize}

	\subsubsection{Privacy Analysis of LP-LDA}
	\begin{theorem}
		The LP-LDA satisfies $\varepsilon\text{-differential privacy}$ for each document contributor where $\varepsilon=\ln\frac{1-f/2}{f/2}$.
	\end{theorem}
	\begin{proof}
		Suppose a word $t$ appears in a noisy bit vector, then the probability of it being kept from the original bit vector is $Pr(\hat{\mathbf{V}}_m [t]=1|\mathbf{V}_m [t]=1)=1-f/2$ and the probability of it being flipped from the original bit vector is $Pr(\hat{\mathbf{V}}_m [t]=1|\mathbf{V}_m [t]=0)=f/2$. Then, according to the definition of DP, it guarantees the privacy of
		\begin{equation*}
		\begin{aligned}
		\varepsilon=\left|\ln\frac{Pr(\hat{\mathbf{V}}_m [t]=1|\mathbf{V}_m [t]=1)}{Pr(\hat{\mathbf{V}}_m [t]=1|\mathbf{V}_m [t]=0)}\right|=\ln\frac{1-f/2}{f/2}.
		\end{aligned}
		\end{equation*}
		The analysis also holds for any bit $t$ that $\hat{\mathbf{V}}_m [t]=0$.
	\end{proof}
	Since the reconstruction and training process are essentially post-processes on the noisy bit vectors, the local privacy remains unchanged for all the documents.
	
	\subsubsection{Utility Analysis of LP-LDA}
	\begin{theorem}
		Let $N_t$ and $n_t$ denote the counts of word $t$ in the original and perturbed datasets, respectively, then
		\begin{equation}
		\hat{N_{t}}=\frac{2n_t-fM}{2(1-f)}
		\end{equation}
		is an unbiased estimator of $N_t$ with the variance of
		\begin{equation}
		D(\hat{N_{t}})=\frac{(2-f)fM}{4(1-f)^2}.
		\end{equation}
	\end{theorem}
	\begin{proof}
		Let $n_1$ denote the count of word $t$ retained from the real datasets and $n_2$ denote the noisy part, then $n_1$ and $n_2$ follow two Binomial distributions, i.e.,
		$n_{1}\sim B(N_{t}, 1-f/2)$, $n_{2}\sim B(M-N_{t}, f/2)$. Let $ X=n_{1}+n_{2}$, then its  first theoretical moment $E(X)=N_{t}(1-f/2)+(M-N_{t})\cdot(f/2)$ and its first sample moment $\bar{X}= n_{t}$. Therefore,
		\begin{equation*}
		\hat{N_{t}}=\frac{2n_t-fM}{2(1-f)}
		\end{equation*}
		is the moment estimator as well as unbiased estimator. Its variance is then
		\begin{equation*}
		\begin{aligned}
		D(\hat{N_{t}})&=\frac{var(n_t)}{(1-f)^2}=\frac{ var(n_{1}+n_{2})}{(1-f)^2}=\frac{(2-f)fM}{4(1-f)^2}.
		\end{aligned}
		\end{equation*}
	\end{proof}

	\section{Experiment}
	In this section, we evaluate the effectiveness of our proposed privacy monitoring algorithm and locally private LDA algorithm LP-LDA on real-world datasets. 
	
	The datasets used in our experiment are: 
	\textsf{KOS}\footnote{http://archive.ics.uci.edu/ml/}: contains 3430 blog entries from dailykos website.
	\textsf{NIPS}\footnote{http://nips.djvuzone.org/txt.html}: contains 1740 research papers from NIPS conference. 
	\textsf{Enron}\footnote{www.cs.cmu.edu/~enron}: contains 0.5 million email messages from about 150 users.  
	
	We extracted part of these datasets as our training datasets and the rest as the testsets. For simplicity, we setup a pre-processing phase on these dataset before running our experiments. For example, all stop words were removed and 1000 most frequent words in each dataset were chosen as the corresponding vocabulary list. Details about these datasets after pre-processing can be found in Table 1.
	
	In our experiments, for all datasets, the topic number is set as $50$, the maximum iteration number of CGS process in LDA model training is set as $300$, which is sufficient for convergence on all three datasets. The hyper parameters $\alpha$ and $\beta$ are set as 0.1, 0.01, respectively.
	
	\begin{table}
		\centering
		\begin{tabular}{cccc}
			\hline
			Dataset  & $\#.$ words  & $\#.$ training docs & $\#.$ test docs\\
			\hline
			\ KOS       & 209169     &3000 &430 \\
			\ NIPS         &410753       &1349 &150\\
			\ Enron         &356363       &8000 &2000\\
			\hline
		\end{tabular}
		\caption{Details about the real-world datasets}
		\label{tab:plain}
	\end{table}

	\subsection{Inherent privacy of CGS in LDA}
	
	\begin{figure}[htbp]
		\centering
		\subfigure[Document-level privacy]{
			\includegraphics[width=4.0cm]{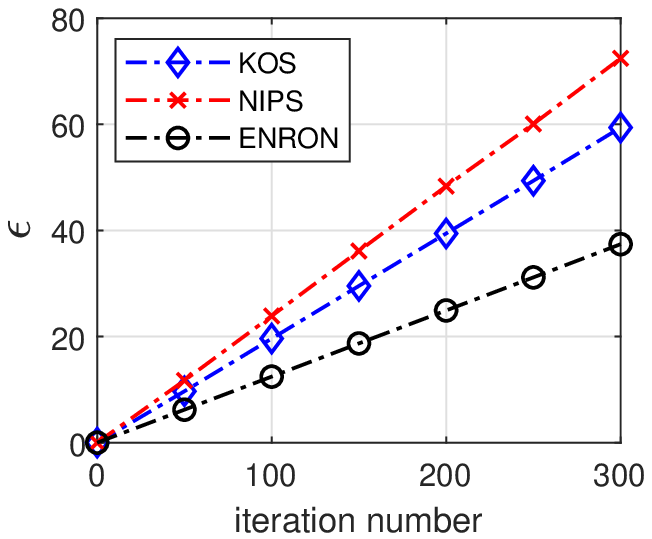}
		}
		\subfigure[Word-level privacy]{
			\includegraphics[width=4.0cm]{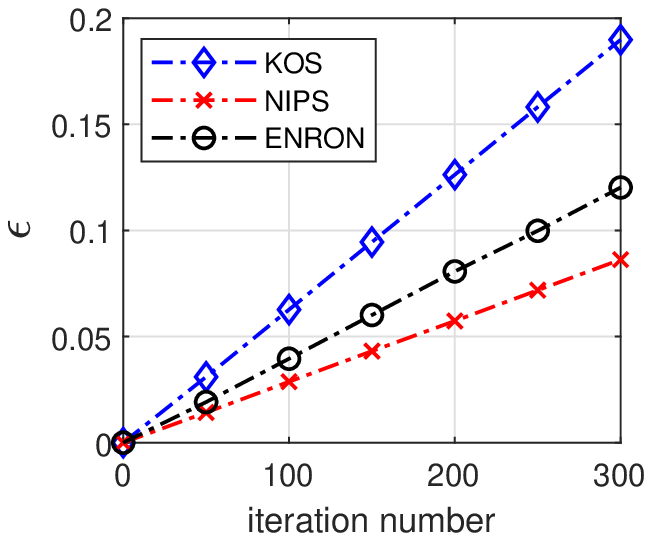}
		}
		
		\caption{privacy guarantee vs. iteration number of CGS in LDA}
		\label{fig:privacy monitoring}
	\end{figure}
	
	\begin{figure*}[htb]
		\centering
		\subfigure[\textsf{KOS}]{
			\includegraphics[width=5.5cm]{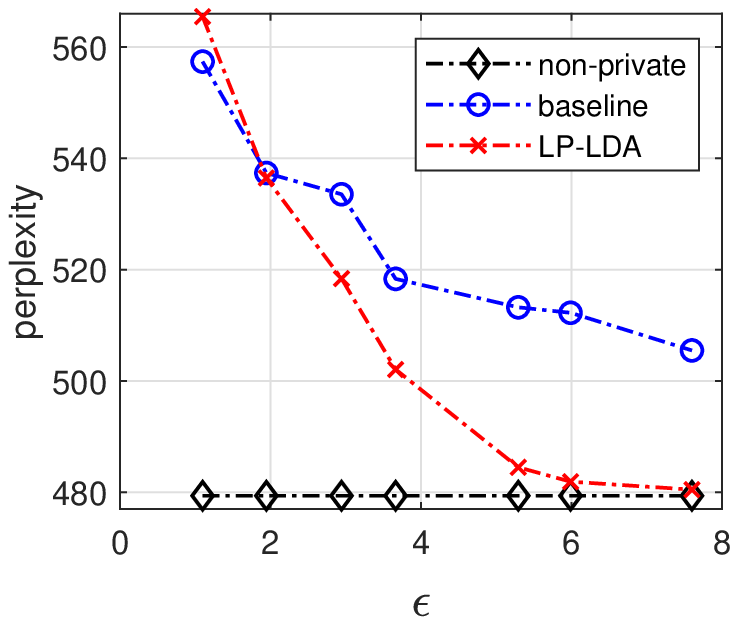}
			\label{benefits for staleness: learning rate=01}
		}
		\subfigure[\textsf{NIPS}]{
			\includegraphics[width=5.5cm]{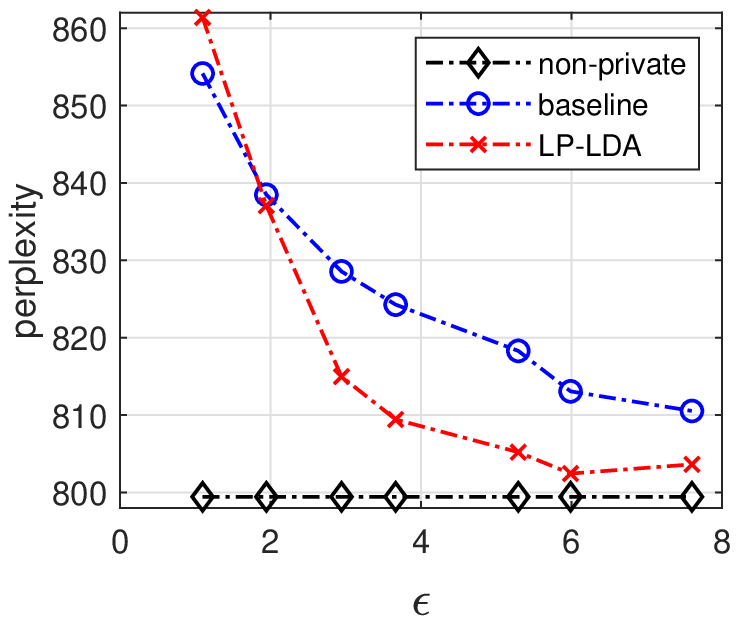}
		}
		\subfigure[\textsf{Enron}]{
			\includegraphics[width=5.5cm]{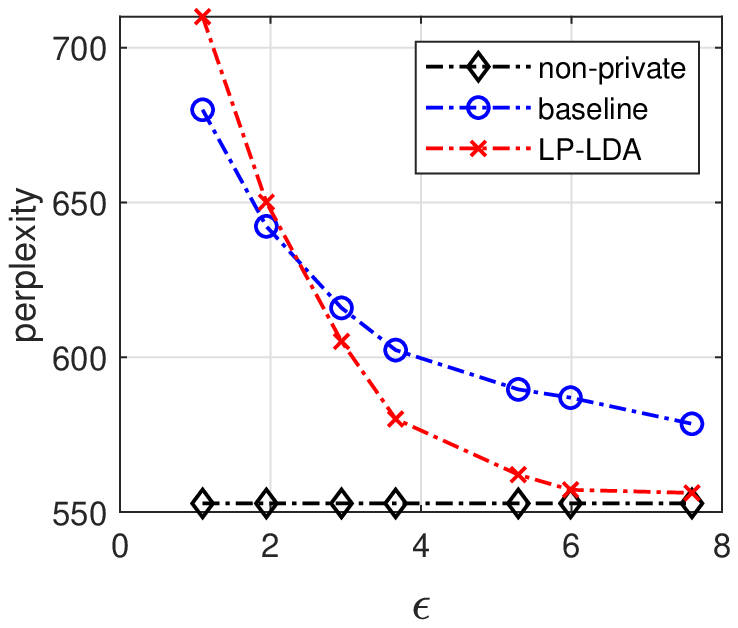}
		}
		\caption{Perplexity of LDA model vs. privacy }
		\label{fig:LP-LDA}
	\end{figure*}

	Figure~\ref{fig:privacy monitoring} illustrates the inherent privacy guarantee of CGS algorithm in LDA measured by our proposed privacy monitoring algorithm on three datasets for both document-level and word-level privacy. It should be noted that a larger privacy parameter $\varepsilon$ in the figures means less privacy guarantee.
	
	As we can see in both subfigures, both word-level and document-level privacy parameter $\varepsilon$ of CGS in LDA increase approximately linearly with the number of sampling iterations. This is because the privacy bound in each iteration of sampling is very close, and the total privacy parameter will accumulate with the number of iterations according to the sequential composition theorem. 
	
	Although CGS on all datasets can obtain privacy guarantee for free, the inherent privacy varies on different datasets. For document level, the privacy guarantee achieved on \textsf{NIPS} is the weakest while that on \textsf{Enron} is the strongest. That is because the documents in \textsf{NIPS} averagely contain the most words, which also means it is the most difficult to be effectively hidden. For word-level privacy, the LDA model trained on \textsf{NIPS} has the strongest privacy guarantee because it contains largest number of words and the sampling probability for each unique word will be the lowest. On the contrary, with the fixed length of vocabulary list, \textsf{KOS} contains the fewest words in total and results the weakest word-level privacy after same number of iteration. 
	
	\subsection{Local mechanism}
	Figure~\ref{fig:LP-LDA} depicts the simulation performance of our proposed LP-LDA mechanism in terms of different level of privacy. The flipping probability $f$ in LP-LDA varies from $0.5$ to $0.001$, and the corresponding privacy level varies from $1.089$ to $7.6004$. The utility of LDA model training is measured by the perplexity on test sets. Perplexity is an information-theoretical measure commonly used to evaluate the prediction performance of LDA model and generally smaller perplexity on a test set means better prediction accuracy. In particular, we compared LP-LDA with a baseline privacy-preserving LDA mechanism based on Laplace mechanism, in which the sufficient statistics of the likelihood, i.e., word count matrices $N_k^t$ and $N_m^k$ are privatized at the beginning of the CGS algorithm with the sensitivity of $1$ and privacy of $\varepsilon$~\cite{foulds2016theory}.

	As shown, both the perplexity of LP-LDA and baseline algorithm decrease with the increase of $\varepsilon$, which shows the trade-off between the privacy and utility. For stronger privacy regime with smaller $\varepsilon$, the perplexity of LP-LDA is larger than that of the baseline algorithm. That is because the Laplace mechanism baseline algorithm incurs less noise than randomized response in LP-LDA for the statistics of word count $N_t$. While for weaker privacy regime with larger $\varepsilon$, the perplexity of LP-LDA is far less than that of the baseline algorithm and shows greater LDA model training utility. These utility comparison results can be also explained by the variance difference of the word count $N_t$ in two mechanisms. In baseline mechanism based on Laplace noise, the noise variance is $D(N'_{t})=2K^{2}/\varepsilon^{2}$, while the variance $D(\hat{N}_t)$ in our proposed LP-LDA is shown in Equation (10). In particular, for larger $\varepsilon$ on all three datasets, we can always have $D({\hat{N}_{t}})<D(N'_{t})$.

	\section{Conclusion and future work}
	In this work, we investigate the privacy protection of LDA model training. We first present that the CGS algorithm in LDA can possess some inherent privacy in each sampling process and then propose a efficient searching-based privacy monitoring algorithm to identify the privacy guarantee bound in the iterative CGS process of LDA. In addition, besides training on a trustworthy data server, we also propose a locally private solution of LP-LDA to achieve LDA training on a sanitized dataset by individual local users, which is applicable to many scenarios. The experiments on real-world datasets validate our proposed approaches. Future work will center on finding tighter bound of the inherent privacy guarantee in LDA model training. 
	
	\section{Acknowledgments}
	
	This work is supported by National Natural Science Foundation of China under Grant 61572398, Grant 61772410, Grant 61802298, Grant 11690011 and Grant U1811461, in part by the Fundamental Research Funds for the Central Universities under Grant xjj2018237, in part by the China Postdoctoral Science Foundation under Grant 2017M623177, and in part by the National Key Research and Development Program of China under Grant 2017YFB1010004.

\bibliographystyle{named}
\bibliography{ijcai19}
\begin{appendices}\appendixtitleon
		\appendix
		
		\section{Proof of Theorem 1}\label{proofth1}
		\begin{proof}
		The sensitivity obtained from topic \emph{j} can be computed as
		\begin{equation*}
		\xi_{j}=|\ln\frac{p_{j}^{'}}{p_j}|=|\ln(\frac{\sum_{k}r_{k}^{'}}{\sum_{k}r_k}\cdot\frac{\sum_{t}(n_{j}^{t}+\beta)-N_j}{\sum_{t}(n_{j}^{t}+\beta)})|
		\end{equation*}
		then we compare $\xi_{j}$ and $\xi_{k}$
		\begin{equation*}
		\begin{aligned}
		\xi_{k}-\xi_{j}&=|\ln\frac{\sum_{k}r_{k}^{'}}{\sum_{k}r_{k}}+\ln\frac{\sum_{t}(n_{k}^{t}+\beta)-N_k}{\sum_{t}(n_{k}^{t}+\beta)}|\\&-
		|\ln\frac{\sum_{k}r'_{k}}{\sum_{k}r_{k}}+\ln\frac{\sum_{t}(n_{j}^{t}+\beta)-N_j}{\sum_{t}(n_{j}^{t}+\beta)}|
		\end{aligned}
		\end{equation*}
		By condition (4), it's easy to prove that
		\begin{equation}
		\begin{aligned}
		|\ln\frac{\sum_{k}r_{k}^{'}}{\sum_{k}r_{k}}+\ln\frac{\sum_{t}(n_{j}^{t}+\beta)-N_k}{\sum_{t}(n_{j}^{t}+\beta)}|<\ln\frac{\sum_{k}r_{k}^{'}}{\sum_{k}r_k}
		\end{aligned}
		\end{equation}
		We can observe that $\xi_{k}=\xi_{j}$ when $N_k=N_j=0$, and $\xi_{k}>\xi_{j}$ when $N_k=0,N_{j}\neq 0$ and condition (11) holds. Hence, $\xi_k=\max\{\xi_1\,\xi_2\,...,\xi_K\}$ holds due to the arbitrariness of \emph{k} and \emph{j}.
	\end{proof}
	
	\section{Proof of Lemma 1}\label{prooflemma}
	
	\begin{proof}
		For convenience, we denote
		\begin{equation*}
		\frac{n_{k,\neg i}^{t}+\beta}{\sum_{t=1}^{V}(n_{k,\neg i}^{t}+\beta)-N_k}\cdot\frac{n_{m,\neg i}^{k}+\alpha}{\sum_{k=1}^{K}(n_{m,\neg i}^{k}+\alpha)}\quad by\quad \frac{a_k}{b_k-N_k}
		\end{equation*}
		where
		\begin{equation*}
		a_{k}=(n_{k,\neg i}^{t}+\beta)\cdot\frac{n_{m,\neg i}^{k}+\alpha}{\sum_{k=1}^{K}(n_{m,\neg i}^{k}+\alpha)}\quad b_{k}=\sum_{t=1}^{V}(n_{k,\neg i}^{t}+\beta)
		\end{equation*}
		then the problem is transformed into proving that
		\begin{equation}
		\sum_{j\neq k}\frac{a_j}{b_j}+\frac{a_k}{b_k-N}>\sum_{j=1}^{K}\frac{a_j}{b_j-N_j}
		\end{equation}
		holds if
		\begin{equation*}
		\frac{a_k}{b_k-N}=\max\{\frac{a_j}{b_j-N},j\in\{1,2,...,K\}\}
		\end{equation*}
		inequality (12) is equivalent to
		\begin{equation}
		\frac{a_k}{b_k-N}-\frac{a_k}{b_k}>\sum_{j=1}^{K}(\frac{a_j}{b_j-N_j}-\frac{a_j}{b_j})
		\end{equation}
		To prove inequality (13), we consider a function set
		\begin{equation*}
		\mathcal{Y}=\{y_{j}(x),j\in\{1,2,...,K\}\}
		\end{equation*}
		where
		\begin{equation*}
		y_{j}(x)=\frac{a_j}{b_j-x}-\frac{a_j}{b_j}
		\end{equation*}
		then each function in $\mathcal{Y}$ is determined by a pair of parameters $(a_j,b_j)$. consider the relation between $(a_j,b_j)$ and $(a_i,b_i)$ where $i,j\in\{1,2,...,K\}$, it must belongs to one of two cases below:
		\begin{equation*}
		case1:a_i\geq a_j,b_i\leq b_j\quad case2:a_i<a_j,b_i> b_j
		\end{equation*}
		\emph{Case1}:It must holds that $y_{i}(N_i)+y_{j}(N_j)<y_{i}(N_i)+y_{i}(N_j)<y_{i}(N_i+N_j)$ since $y_{i}(0)=y_{j}(0)$ and $y_{i}^{'}(x)>y_{j}^{'}(x),\forall x>0$\\
		\emph{Case2}:It must holds that
		\begin{equation*}
		y_{i}(N_i)+y_{j}(N_j)<\max\{y_{i}(N_i+N_j),y_{j}(N_j+N_i)\}
		\end{equation*}
		In fact, it is easy to prove that there exists only one intersection in $(0,\min\{b_i,b_j\})$  between $y_{i}(x)\quad and\quad y_{j}(x)$, denoted by $(x^*,y^*)$. Based on this, the distribution of $N_i,N_j,N_i+N_j$ on number axis also has three cases to consider:\\
		
		\textbf{case1:}  $N_j\leq x^*,N_i+N_j>x^*$, by computing the derivatives of $y_{i}^{'}(x)$ and $y_{j}^{'}(x)$, we have
		\begin{equation*}
		y_{i}(N_i)+y_{j}(N_j)<\max\{y_{i}(N_i+N_j),y_{j}(N_j+N_i)\}
		\end{equation*}
		
		\textbf{case2:} $N_{j}> x^{*},N_{i}+N_{j}>x^{*}$, since $y_{j}(N_j)<y_{i}(N_j)$  holds, then:
		\begin{equation*}
		y_{i}(N_i)+y_{j}(N_j)<\max\{y_{i}(N_i+N_j),y_{j}(N_j+N_i)\}
		\end{equation*}
		
		\textbf{case3:}$N_i+N_j\leq x^{*}$, since $y_{i}(N_i)<y_{j}(N_i)$  holds, then:
		\begin{equation*}
		y_{i}(N_i)+y_{j}(N_j)<\max\{y_{i}(N_i+N_j),y_{j}(N_j+N_i)\}
		\end{equation*}
		
		We have proved that there must exists function $y_{i}(x)$ such that $y_i(N)>\sum_{j=1}^{K}y_j(N_j)$ through the research above on the property of $y_{i}(x)$ . So far, Lemma 1 has been proved.
	\end{proof}
	
	\section{Proof of Theorem 2}\label{proofth2}
	
	\begin{proof}
		Given a partition $\gamma^{*}=\{N_1,...N_K\}$ not in $\mathcal{P}$, it suffices to verify that there exist some partitions from $\mathcal{P}$ such that the privacy parameter obtained from $\gamma^{*}$ is smaller than parameter from those partitions.  Assume that the privacy parameter from $\gamma^{*}$ is $\varepsilon=2\xi=2\max\{\xi_1\,\xi_2\,...,\xi_K\}$ ,then there are two cases need to be considered:\\
		\begin{equation*}
		\begin{aligned}
		case1:\xi<\ln\frac{\sum_{k}r_{k}^{'}}{\sum_{k}r_k}\quad case2:\xi>\ln\frac{\sum_{k}r_{k}^{'}}{\sum_{k}r_k}
		\end{aligned}
		\end{equation*}
		\emph{Case 1:}Based on lemma 1, there exists a partition $\gamma^{'}$ from $\mathcal{P}$ such that $\sum_{k}r_{k}^{'}=\max_{\gamma}\{\sum_{k}r_{k}^{'}|\gamma\}$ ,and since there exists $N_j=0$ in $\gamma^{'}$, then according to corollary 1
		\begin{equation*}
		\begin{aligned}
		\varepsilon_{\gamma^{'}}\geq \ln\frac{\sum_{k}r_{k}^{'}}{\sum_{k}r_k}|_{\gamma^{'}}\geq \ln\frac{\sum_{k}r_{k}^{'}}{\sum_{k}r_k}|_{\gamma^{*}}
		\end{aligned}
		\end{equation*}
		\emph{Case 2:}Based on theorem 1, elements in $\{\xi_1\,\xi_2\,...,\xi_K\}$ satisfy
		\begin{equation*}
		\xi_{j}=|\ln(\frac{\sum_{k}r_{k}^{'}}{\sum_{k}r_{k}}\cdot\frac{\sum_{t}(n_{j}^{t}+\beta)-N_j}{\sum_{t}(n_{j}^{t}+\beta)})|
		\end{equation*}
		Since $N_j\ll\sum_{t} n_{j}^{t}$ always holds, especially for a large corpus, then it's not hard to deduce that
		\begin{equation*}
		\xi_{j}<\ln\frac{\sum_{k}r_{k}^{'}}{\sum_{k}r_k}|_{\gamma'}
		\end{equation*}
		until now, we have proved the existence of the $\gamma^{'}$. According to theorem 1, if condition (4) holds for each $k\in\{1,2,...,K\}$ with $N_{k}=N$, then equation (8) will hold directly. So far, theorem 2 has been proved completely.
	\end{proof}
	
\end{appendices}

\end{document}